\documentclass{llncs}

\usepackage{graphicx}
\usepackage{subfigure} 
\usepackage{amsmath, amssymb, amsfonts}
\usepackage{xfrac}
\usepackage[sectionbib,numbers]{natbib}
\usepackage{algorithm}
\usepackage[noend]{algorithmic}
\usepackage{subfigure}
\usepackage{url}
\usepackage{float}
\restylefloat{table}
\usepackage{color}\definecolor{darkblue}{rgb}{0,0,0.75}
\usepackage[pdftex,pdftitle="Effective Parallelisation of Machine Learning Algorithms",colorlinks=true, bookmarksnumbered, citecolor=darkblue, urlcolor=darkblue, linkcolor=darkblue, pdfpagemode=UseNone]{hyperref}
\usepackage[misc]{ifsym}

\newtheorem{pr}{Proposition}
\newtheorem{thm}[pr]{Theorem}
\newtheorem{lm}[pr]{Lemma}
\newtheorem{df}[pr]{Definition}
\newtheorem{crl}[pr]{Corollary}

\newcommand{\defemph}[1]{\textbf{#1}}

\newcommand{\innerprod}[2]{\langle #1,#2\rangle}

\newcommand{\norm}[1]{\left\|#1\right\|}
\newcommand{\mapping}[3]{#1\!: #2 \to #3}
\newcommand{\R}{\mathbb{R}}
\newcommand{\N}{\mathbb{N}}

\newcommand{\bigO}[1]{\mathcal{O}\left( #1 \right)}

\newcommand{\loss}{\ell}
\newcommand{\cummloss}[2]{L_{#1}(#2)}
\newcommand{\lossbound}[2]{\textbf{L}_{#1}(#2)}

\newcommand{\cummcomm}[2]{C_{#1}(#2)}

\newcommand{\commcost}[1]{c_{#1}}

\newcommand{\round}{t}
\newcommand{\totalRounds}{T}
\newcommand{\learner}{i}
\newcommand{\totalLearners}{m}

\newcommand{\uprule}{\varphi}
\newcommand{\syncop}{\sigma}
\newcommand{\divergence}{\delta}
\newcommand{\divThreshold}{\Delta}
\newcommand{\compErr}{\epsilon}

\newcommand{\hilbertSpace}{\mathcal{H}}
\newcommand{\kernel}{k}
\newcommand{\model}{f}
\newcommand{\modelconf}{\mathbf{\model}}
\newcommand{\avgmodel}{\overline{\modelconf}}
\newcommand{\SVs}{S}
\newcommand{\coeffs}{\alpha}
\newcommand{\modelApprox}{\widetilde{\model}}
\newcommand{\refModel}{r}

\newcommand{\onlineAlgo}{\mathcal{A}}
\newcommand{\protocol}{\Pi}
\newcommand{\dynProt}{\mathcal{D}}
\newcommand{\periodProt}{\mathcal{P}}
\newcommand{\contProt}{\mathcal{C}}

\newcommand{\dist}{P} 

\newcommand{\sampleSpace}{X}
\newcommand{\outputSpace}{Y}
\newcommand{\sample}{x}
\newcommand{\truelabel}{y}
\newcommand{\inputSpace}{\sampleSpace\times\outputSpace}
\newcommand{\inputval}[2]{\left( \sample_{#1}^{#2}, \truelabel_{#1}^{#2} \right)}

\title{Communication-Efficient Distributed Online Learning with Kernels}
\date{}
\author{
Michael Kamp\inst{1}{~\tiny\Letter} \and
Sebastian Bothe\inst{1} \and
Mario Boley\inst{2} \and 
Michael Mock\inst{1}
}
\institute{Fraunhofer IAIS,~Sankt Augustin, Germany\\\email{<name>.<surname>@iais.fraunhofer.de}
\and
Saarland University, \email{mboley@mmci.uni-saarland.de}
}

\begin{document}

\maketitle

\begin{abstract}
We propose an efficient distributed online learning protocol for low-latency real-time services. It extends a previously presented protocol to kernelized online learners that represent their models by a support vector expansion. While such learners often achieve higher predictive performance than their linear counterparts, communicating the support vector expansions becomes inefficient for large numbers of support vectors. The proposed extension allows for a larger class of online learning algorithms---including those alleviating the problem above through model compression.
In addition, we characterize the quality of the proposed protocol by introducing a novel criterion that requires the communication to be bounded by the loss suffered.
\end{abstract}

\section{Introduction}
We consider the problem of distributed online learning for low-latency real-time services~\cite{dekel/jmlr/2012,kamp2014communication}. In this scenario, a learning system of $\totalLearners\in\N$ connected local learners provides a real-time prediction service on multiple dynamic data streams. 
In particular, we are interested in generic distributed online learning protocols that treat concrete learning algorithms as a black-box. 
The goal of such a protocol is to provide, in a communication efficient way, a service quality similar to a serial setting in which all examples are processed at a central location.
While such an optimal predictive performance can be trivially achieved by centralizing all data, the required continuous communication usually exceeds practical limits (e.g., bandwidth constraints~\cite{barroso2013datacenter}, latency~\cite{yuan2013real, heires2010budgeting}, or battery power~\cite{predd2007distributed, deligiannakis2008bandwidth}). 
Similarly, communication limits can be satisfied trivially by letting all local learners work in isolation. However, this usually comes with a loss of service quality that increases with the number of local learners.

In previous work, we presented a protocol that effectively reduces communication while providing strict loss bounds for a class of algorithms that perform loss-proportional convex updates of linear models~\cite{kamp2014communication}. 
That is, algorithms that update linear models in the direction of a convex set with a magnitude proportional to the instantaneous loss (e.g., Stochastic Gradient Descent~\cite{boyd2004convex}, or Passive Aggressive~\cite{crammer2006online}).
The protocol is able to cease communication as soon as no loss is suffered anymore. However, for most realistic problems this cannot be achieved by linear models. Thus, a more complex hypothesis class is desirable that enables the learners to achieve zero loss and thus reach quiescence. 

Kernelized online learning algorithms can provide such an extended hypothesis class, but practical versions of these algorithms do not perform loss-proportional convex updates (e.g.,~\cite{kivinen2004online,orabona2009bounded,wang2010online}).
Therefore, in this paper we extend the class of algorithms to \emph{approximately} loss-proportional convex updates (Sec.~\ref{sec:dynProtocol}). This relaxation is particularly crucial for kernelized online learners for streams that represent the model by its support vector expansion. These learners use this relaxation in order to reduce the number of support vectors, since otherwise a monotonically increasing model size would render them infeasible in streaming settings.

Also, for the first time we characterize the quality of the proposed protocol by introducing a novel criterion for efficient protocols that requires a strict loss bound and ties the loss to the allowed amount of communication. In particular, the criterion implies that the communication vanishes whenever the loss approaches zero. We bound the loss and communication of the proposed protocol and show for which class of learning algorithms it fulfills the efficiency criterion (Sec.~\ref{sec:perfGuarantees}).
\begin{figure}[t!]
	\centering
	\subfigure[]{\label{fig:xorExampleError}\includegraphics[width=6.0cm]{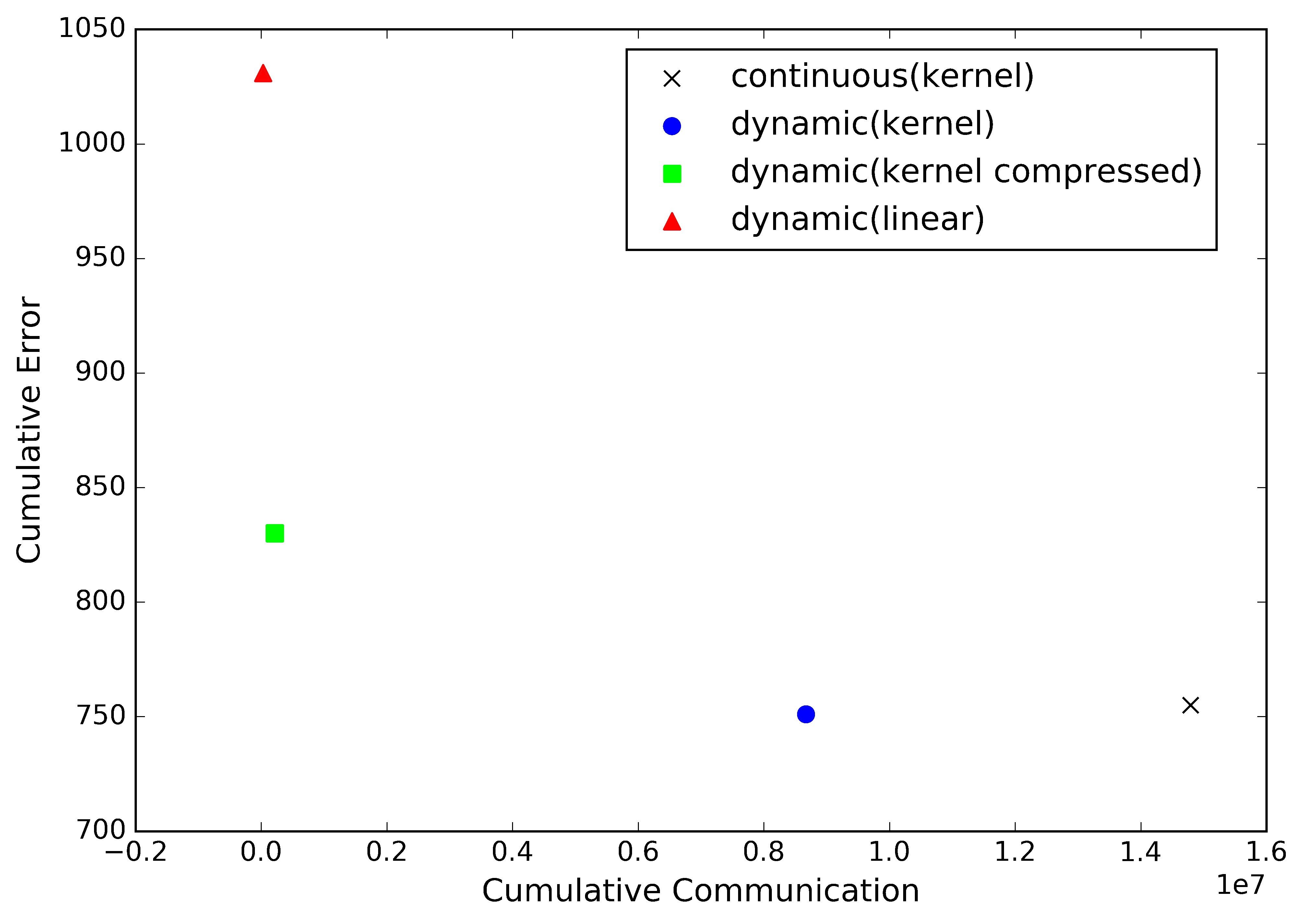}}\hfill%
	\subfigure[]{\label{fig:xorExampleComm}\includegraphics[width=6.0cm]{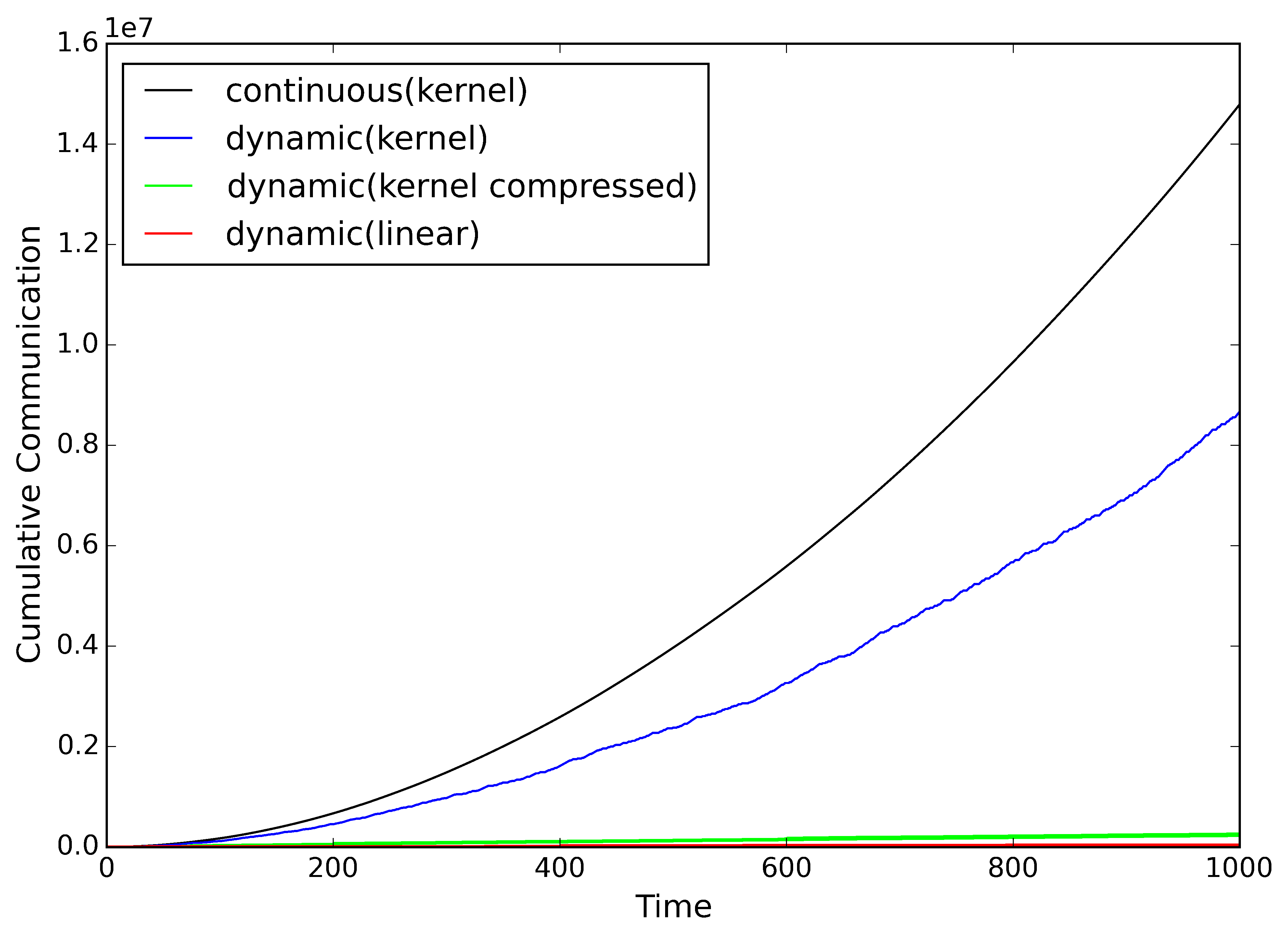}}\hfill%
	\caption{
		\subref{fig:xorExampleError} Trade-off between cumulative error and cumulative communication, and \subref{fig:xorExampleComm} cumulative communication over time of a distributed learning system using the proposed protocol. The learning task is classifying instances from the UCI SUSY dataset with $4$ learners, each processing $1000$ instances. Parameters of the learners are optimized on a separate set of $200$ instances per learner. 
	}		
	\label{fig:xorExample}
\end{figure}
While the strict loss bound required in our criterion can be achieved by periodically communicating protocols~\cite{mcdonald2009efficient, dekel/jmlr/2012}, their communication never vanishes, independent of their loss, which is also required for efficiency. 
By communicating only when it significantly improves the service quality, our protocol achieves similar service quality as any periodically communicating protocol while communicating less by a factor depending on its in-place loss. 

To further amplify this advantage, we apply methods from serial kernelized in-stream learning approaches. These approaches reduce the number of support vectors, e.g., by truncating individual support vectors with small weights~\cite{kivinen2004online}, or by projecting a single support vector on the span of the remaining ones~\cite{orabona2009bounded, wang2010online}. 
%

We illustrate the impact of the choice of the hypothesis class on the predictive performance and communication as well as the impact of model compression on an example dataset in Fig.~\ref{fig:xorExample}. In this example, we predicted the class of instances drawn from the SUSY dataset from the UCI machine learning repository~\cite{Lichman:2013}. The learning systems using linear models continuously suffer loss resulting in a large cumulative error, but since the linear models are small compared to support vector expansions, the cumulative communication is small. A continuously synchronizing protocol using support vector expansions has a significantly smaller loss at the cost of very high communication, since each synchronization requires to send models with a growing number of support vectors. Using the proposed dynamic protocol, this amount of communication can be reduced without losing in prediction quality. In addition, when using model compression the communication can be further reduced to an amount similar to the linear model, but at the cost of prediction quality. 

We further discuss the behavior of our protocol with respect to the trade-off between predictive performance and communication, and point out the strengths and weaknesses of the protocol in Sec.~\ref{sec:discussion}.
%

%

\section{Distributed Online Learning with Kernels}
\label{sec:dynProtocol}
In this section, we provide preliminaries and describe the protocol, extend it from linear function spaces to kernel Hilbert spaces, and provide an effectiveness criterion for distributed online learning. For that, we consider \defemph{distributed online learning protocols} $\protocol=(\onlineAlgo,\syncop)$ that run an online learning algorithm $\onlineAlgo$ on a distributed system of $\totalLearners\in\N$ local learners and exchange information between these learners using a synchronization operator $\syncop$. 

\paragraph{Preliminaries:}
The \defemph{online learning algorithm} $\onlineAlgo=(\hilbertSpace,\uprule, \loss)$ run at each \defemph{local learner} ${\learner\in[\totalLearners]}$ maintains a \defemph{local model} $\model^{\learner}\in\hilbertSpace$ from a function space $\hilbertSpace$ using an update rule $\uprule$ and a loss function $\loss$. 
That is, at each time point $\round\in\N$, each learner $\learner$ observes an individual input $\inputval{\round}{\learner}$ drawn independently from a time-variant distribution ${\mapping{\dist_\round}{\inputSpace}{[0,1]}}$ over an input space $\inputSpace$ . 
Based on this input and the local model, the local learner provides a service whose quality is measured by the \defemph{loss function} ${\mapping{\loss}{\hilbertSpace\times\inputSpace}{\R_+}}$. 
After providing the service, the local learner updates its local model using the \defemph{update rule} ${\mapping{\uprule}{\hilbertSpace\times\inputSpace}{\hilbertSpace}}$ in order to minimize the cumulative loss. 
%
The \defemph{synchronization operator} $\mapping{\syncop}{\hilbertSpace^\totalLearners}{\hilbertSpace^\totalLearners}$ transfers the current \defemph{model configuration} ${\modelconf=\left(\model^1,\dots,\model^\totalLearners\right)}$ of $\totalLearners$ local models to the synchronized configuration $\syncop(\modelconf)$. In the following, we recapitulate the dynamic protocol presented in~\cite{kamp2014communication} as well as two baseline protocols, i.e., a continuously and a periodic protocol. 

Given an online learning algorithm $\onlineAlgo$, the \defemph{periodic protocol} $\periodProt=(\onlineAlgo,\syncop_b)$ synchronizes every $b\in\N$ time steps the current model configuration ${\modelconf}$ by replacing all local models by their joint \defemph{average} 
$
\avgmodel=\sfrac{1}{\totalLearners}\sum_{\learner=1}^{\totalLearners}\model^\learner
$. That is, the synchronization operator is given by
\[
\syncop_b(\modelconf_t)=	\begin{cases}
									\left(\avgmodel_\round,\dots,\avgmodel_\round\right), &\text{ if } b \mid \round\\
									\enspace\modelconf_\round = (\model_\round^1,\dots,\model_\round^\totalLearners), &\text{ otherwise}\\
								\end{cases}\enspace .
\]
A special case of this is the \defemph{continuous protocol} $\contProt=(\onlineAlgo,\syncop_1)$ that continuously synchronizes every round, i.e., $
\syncop_1\left(\modelconf\right)=\left(\avgmodel,\dots,\avgmodel\right)
$.

The \defemph{dynamic protocol} $\dynProt=(\onlineAlgo,\syncop_\divThreshold)$ synchronizes the local learners using a \defemph{dynamic operator} $\syncop_{\divThreshold}$~\cite{kamp2014communication}. This operator only communicates when the \defemph{model divergence} 
\begin{equation}
\divergence(\modelconf)=\frac{1}{\totalLearners}\sum_{\learner=1}^{\totalLearners}\norm{\model^{\learner}-\avgmodel}^2
\label{eq:divergence}
\end{equation}
exceeds a \defemph{divergence threshold} $\divThreshold$. 
That is, the dynamic averaging operator is defined as
\[
\syncop_\Delta(\modelconf_\round)=\begin{cases}(\overline{\modelconf}_\round,\dots,\overline{\modelconf}_\round), &\text{ if } \divergence(\modelconf_\round)>\divThreshold\\
\modelconf_\round, &\text{ otherwise}\\
\end{cases}\enspace .
\]
In order to decide when to communicate, each local learner $\learner\in [\totalLearners]$  monitors the \defemph{local condition} $\|\model_{\round}^{\learner}-\refModel_\round\|^2\leq\Delta$ for a \defemph{reference model} $\refModel_\round\in\hilbertSpace$ that is common among all learners (see~\cite{keren2012shape,sharfman/tods/2007,gabel/IPDPS/2014,giatrakos2012prediction} for a more general description of this method). The local conditions guarantee that if none of them is violated, the divergence does not exceed the threshold $\divThreshold$. 
The closer the reference model is to the true average of local models, the tighter are the local conditions. 
Generally, the first choice for the reference model is the average model from the last synchronization step. Note, however, that there are several refinements of this choice that can be used in practice to further reduce communication.
%
%

\paragraph{Efficiency Criterion:}
In the following, we introduce performance measures in order to analyze the dynamic protocol and compare it to the continuous and periodic protocols. We measure the predictive performance of a distributed online learning system until time $\totalRounds\in\N$ by its cumulative loss
\[
\cummloss{}{\totalRounds,\totalLearners}=\sum_{\round=1}^\totalRounds\sum_{\learner=1}^{\totalLearners} \loss(\model_{\round}^{\learner},\inputval{\round}{\learner})\enspace .
\]
Performance guarantees are typically given by a \defemph{loss bound} $\lossbound{}{\totalRounds,\totalLearners}$, i.e., for all possible input sequences it holds that $\cummloss{}{\totalRounds,\totalLearners}\leq \lossbound{}{\totalRounds,\totalLearners}$. These bounds can be defined with respect to a sequence of reference models, in which case they are referred to as (shifting) \defemph{regret bounds}.

We measure its performance in terms of communication by its cumulative communication
\[
\cummcomm{}{\totalRounds,\totalLearners}=\sum_{\round=1}^{\totalRounds}\commcost{}(\modelconf_{\round})\enspace ,
\]
where $\mapping{\commcost{}}{\hilbertSpace^\totalLearners}{\N}$ measures the number of bytes required by the learning protocol to synchronize models $\modelconf_\round=\left(\model^1_\round,\dots,\model^\totalLearners_\round\right)$ at time $\round$.

There is a natural trade-off between communication and loss of a distributed online learning system. On the one hand, a loss similar to a serial setting can be trivially achieved by continuous synchronization. On the other hand, communication can be entirely omitted. The trade-off for these two extreme protocols can be easily determined: if the cumulative loss of an online learning algorithm $\onlineAlgo$ is bounded by $\lossbound{\onlineAlgo}{\totalRounds}$, the loss of a permanently centralizing system with $\totalLearners$ local learners running $\onlineAlgo$ is bounded by $\lossbound{\contProt}{\totalRounds, \totalLearners} = \lossbound{\onlineAlgo}{\totalLearners\totalRounds}$, i.e., the loss bound of a serial online learning algorithm processing $\totalLearners\totalRounds$ inputs. The protocol transmits $\bigO{\totalLearners}$ messages of size up to $\bigO{\totalRounds}$ in every of the $\totalRounds$ points in time. At the same time, the loss of a distributed system without any synchronization is bounded by $\lossbound{}{\totalRounds,\totalLearners} = \totalLearners \lossbound{\onlineAlgo}{\totalRounds}$, whereas the communication is $\cummcomm{}{\totalRounds}=0$.

The communication bound of an adaptive protocol should only depend on $\lossbound{\onlineAlgo}{\totalRounds}$ and not on $\totalRounds$, while at the same time retaining the loss bound of the serial setting. In the following definition we formalize this in order to provide a strong criterion for effectiveness of distributed online learning protocols.
\begin{df}
	A distributed online learning protocol $\protocol=(\onlineAlgo,\syncop)$ processing $\totalLearners\totalRounds$ inputs is \defemph{consistent} if it retains the loss bound of the serial online learning algorithm $\onlineAlgo$, i.e., 
	\[
	\lossbound{\protocol}{\totalRounds,\totalLearners}\in \bigO{ \lossbound{\onlineAlgo}{\totalLearners\totalRounds} }\enspace .
	\]
	The protocol is \defemph{adaptive} if its communication bound is linear in the number of local learners $\totalLearners$ and the loss bound $\lossbound{\onlineAlgo}{\totalLearners\totalRounds}$ of the serial online learning algorithm, i.e.,
	\[
	\cummcomm{\protocol}{\totalRounds,\totalLearners}\in \bigO{ \totalLearners \lossbound{\onlineAlgo}{\totalLearners\totalRounds} }\enspace .
	\]
	\label{def:efficiency}
\end{df}
An \defemph{efficient} protocol is adaptive and consistent at the same time. In the following section we theoretically analyze the performance of the dynamic protocol with respect to this efficiency criterion. 

\paragraph{Extension to Kernel Methods:}
The protocols presented above are defined for models from a Euclidean vector space. In this paper, we generalize $\hilbertSpace$ to be a \defemph{reproducing kernel Hilbert space} ${\hilbertSpace=\{\mapping{\model}{\sampleSpace}{\R} | \model(\cdot) = \sum_{j=1}^{\dim{F}}w_j\Phi_j(\cdot)\}}$ with \defemph{kernel function} $\mapping{\kernel}{\sampleSpace\times\sampleSpace}{\R}$, \defemph{feature space} $F$, and a mapping ${\mapping{\Phi}{\sampleSpace}{F}}$ into the feature space~\cite{scholkopf2001learning}. The kernel function corresponds to an inner product of input points mapped into feature space, i.e., $\kernel(x,x')=\sum_{j=1}^{\dim{F}}\xi_j\Phi_j(x)\Phi_j(x')$ for constants $\xi_1,\xi_2,\dots\in\R$. Thus, we can express the model in its \defemph{support vector expansion}, or dual representation, i.e.,
$
\model(\cdot)=\sum_{\sample\in\SVs}\coeffs_\sample\kernel(\sample,\cdot)
$
with a set of \defemph{support vectors} ${\SVs=\{\sample_1,\dots,\sample_{\left|\SVs\right|}\}\subset\sampleSpace}$ and corresponding \defemph{coefficients} $\coeffs_\sample\in\R$ for all $\sample\in\SVs$. This implies that the linear weights $w=(w_1,w_2,\dots)\in F$ defining $\model$ are given implicitly by $w_i=\sum_{\sample\in\SVs}\xi_i \coeffs_\sample\Phi_i(\sample)$. 
In order to apply the previously defined synchronization protocols to models from a reproducing kernel Hilbert space, we determine how to calculate the average of a model configuration and its divergence. For that, let ${\modelconf=\left(\model^1,\dots,\model^\totalLearners\right)}\subset\hilbertSpace$ be a model configuration with corresponding weight vectors $\left(w^1,\dots,w^\totalLearners\right)\subset F$, where each model $\learner\in[\totalLearners]$ has support vectors $\SVs^\learner=\{\sample^\learner_1,\dots,\sample^\learner_{|\SVs^\learner|}\}\subset\sampleSpace$ and coefficients $\coeffs^\learner_\sample$ for all $\sample\in\SVs^\learner$. The average is given by 
\[
\avgmodel(\cdot)=\frac{1}{\totalLearners}\sum_{\learner=1}^{\totalLearners}\model^\learner(\cdot)= \frac{1}{\totalLearners}\sum_{\learner=1}^{\totalLearners}\sum_{j=1}^{\dim{F}}w^\learner_j\Phi_j(\cdot) = \frac{1}{\totalLearners}\sum_{\learner=1}^{\totalLearners}\sum_{j=1}^{\dim{F}}\sum_{\sample\in\SVs^\learner}\xi_j\coeffs^\learner_\sample\Phi_j(\sample)\Phi_j(\cdot)\enspace .
\]
We can simplify the above equation to
$
\avgmodel(\cdot) = \frac{1}{\totalLearners}\sum_{\learner=1}^{\totalLearners}\sum_{\sample\in\SVs^\learner}\coeffs^\learner_\sample\kernel(\sample,\cdot)
$. 
By defining the union of support vectors $\overline{\SVs}=\bigcup_{\learner\in [\totalLearners]}S^\learner=\{s_1,\dots,s_{|\overline{\SVs}|}\}$ and augmented coefficients $\overline{\coeffs}^i_{s}\in\R$, which are given by
\[
\overline{\coeffs}^i_{s} = \begin{cases}
							\coeffs_\sample^\learner, & \text{ if }\sample = s \\
							0, &\text{ otherwise}
						 \end{cases}\enspace ,
\]
the dual representation of the average directly follows.
\begin{pr}
For a model configuration ${\modelconf=\left(\model^1,\dots,\model^\totalLearners\right)}\subset\hilbertSpace$, where each model $\learner\in[\totalLearners]$ has augmented coefficients $\overline{\coeffs}^i_{s}$ for $s\in\overline{\SVs}$, the average $\avgmodel\in\hilbertSpace$ is given by 
\[
\avgmodel(\cdot) = \sum_{s\in\overline{\SVs}} \left(\frac{1}{\totalLearners}\sum_{\learner=1}^{\totalLearners}\overline{\coeffs}^\learner_{s}\right) \kernel(s,\cdot)\enspace ,
\]
with support vectors $\overline{S}$ and coefficients $\overline{\coeffs}_s=\sfrac{1}{\totalLearners}\sum_{\learner=1}^{\totalLearners}\overline{\coeffs}^\learner_{s}$ for all $s\in\overline{\SVs}$.
\end{pr}
Using this definition of the average, we now define the distance between models in $\hilbertSpace$ and the divergence $\divergence$ of a model configuration $\modelconf\subset\hilbertSpace$. For an individual model ${\model^\learner}$ and the average ${\avgmodel}$, the distance induced by the inner product of $\hilbertSpace$ is defined by ${\norm{\model^\learner-\avgmodel}=\innerprod{\model^{\learner}}{\model^{\learner}}+\innerprod{\avgmodel}{\avgmodel}-2\innerprod{\model^{\learner}}{\avgmodel}}$, i.e., 
\begin{equation*}
\begin{split}
\norm{\model^\learner-\avgmodel}= \sum_{\sample\in\SVs^\learner} \left(\coeffs_\sample^\learner\right)^2\kernel(\sample,\sample) + \sum_{s\in\overline{\SVs}}\left(\overline{\coeffs}_s\right)^2\kernel(s,s)-2\sum_{\sample\in\SVs^\learner}\sum_{s\in\overline{\SVs}}\coeffs_\sample^\learner\overline{\coeffs}_s\kernel(\sample,s)\enspace .
\end{split}
\end{equation*}
Using this distance, we can compute the divergence (Eq.~\ref{eq:divergence}) for models from a reproducing kernel Hilbert space.

\section{Performance Guarantees}
\label{sec:perfGuarantees}
In order to determine the performance of the dynamic protocol, we start by extending the definition of loss-proportional convex update rules. This allows us to bound the loss for kernelized online learning algorithms that reduce their model size using a compression step.

Let $\mapping{\uprule}{\hilbertSpace\times\inputSpace}{\hilbertSpace}$ be a loss-proportional convex update rule, then $\widetilde{\uprule}$ is an \defemph{approximately loss-proportional convex update rule} if for all $\model\in\hilbertSpace$, $\sample\in\sampleSpace$, and $\truelabel\in\outputSpace$ it holds that
$
\|\widetilde{\uprule}(\model,\sample,\truelabel)-\uprule(\model,\sample,\truelabel)\|\leq\compErr
$.
With this, we can bound the distance between two models after the approximate update step.
\begin{lm}
	For two models $f,g\in\hilbertSpace$ and an approximately loss-proportional convex update rule $\widetilde{\uprule}$, with $\|\widetilde{\uprule}(\model,\sample,\truelabel)-\uprule(\model,\sample,\truelabel)\|\leq\compErr$ for the corresponding loss-proportional convex update rule $\uprule$, it holds that
	\begin{equation*}
	\begin{split}
	\|\widetilde{\uprule}(f,x,y) - \widetilde{\uprule}(g,x,y)\|^2 \leq\|f-g\|^2 -\gamma^2\left(\loss(f,x,y) - \loss(g,x,y)\right)^2 + 2\compErr^2\enspace .
	\end{split}
	\end{equation*}
	\label{lm:updateLemmaWithCompression}
\end{lm}
\begin{proof}
	We abbreviate $\uprule(f,\sample,\truelabel)$ as $\uprule(f)$. Then 
	$\|\widetilde{\uprule}(f)-\uprule(f)\|\leq\compErr$ implies for $f,g\in\hilbertSpace$ that
	${\|\widetilde{\uprule}(f)-\widetilde{\uprule}(g)\|^2\leq\|\uprule(f)-\uprule(g)\|^2+2\compErr^2}$.
	Together with the result from Lm.~4 in~\cite{kamp2014communication}, i.e.,
	$
	{\|\uprule(f) - \uprule(g)\|^2\leq\|f-g\|^2 -\gamma^2\left(\loss(f) - \loss(g)\right)^2}
	$,
	follows the result.
	\qed
\end{proof}
Using Lm.~\ref{lm:updateLemmaWithCompression}, we can bound the loss of our protocol.
\begin{thm}
	Let $\onlineAlgo$ be an online learning algorithm with $\gamma$-loss-proportional convex update rule $\uprule$. Let $\mathbf{d}_1,\dots\mathbf{d}_\totalRounds$ and $\mathbf{p}_1,\dots,\mathbf{p}_\totalRounds$ be two sequences of model configurations such that $\mathbf{d}_1=\mathbf{p}_1$ and the first sequence is maintained by the dynamic protocol $\dynProt=(\onlineAlgo,\syncop_\divThreshold)$ and the second by the periodic protocol $\periodProt=(\onlineAlgo,\syncop_b)$. That is, for $\round=1,\dots,\totalRounds$ the sequence is defined by $\mathbf{d}_{\round+1}=\syncop_\divThreshold\left(\uprule(\mathbf{d}_{\round})\right)$, and $\mathbf{p}_{\round+1}=\syncop_b\left(\uprule(\mathbf{p}_{\round})\right)$ respectively. Then it holds that
	\[
	\cummloss{\dynProt}{\totalRounds,\totalLearners}\leq \cummloss{\periodProt}{\totalRounds,\totalLearners} + \frac{\totalRounds}{\gamma^2}(\divThreshold + 2\compErr^2)\enspace .
	\]
	\label{thm:lossboundModelComp}
\end{thm}
\begin{proof}
	First note that for simplicity we abbreviate $\loss(\model_\round,\sample_\round,\truelabel_\round)$ by $\loss(\model_\round)$.
	We combine our Lm.~\ref{lm:updateLemmaWithCompression} with Lm.~3 from~\cite{kamp2014communication} which states that
	\[
	\frac{1}{\totalLearners}\sum_{\learner=1}^{\totalLearners}\| \syncop_\divThreshold(\mathbf{d})^\learner - \syncop_b(\mathbf{p})^\learner\|^2 \leq \frac{1}{\totalLearners}\sum_{\learner=1}^{\totalLearners}\| d^\learner - p^\learner\|^2 + \divThreshold\enspace .
	\]
	This yields for all $\round\in [\totalRounds]$ that
	\[
	\sum_{\learner=1}^\totalLearners\norm{d_{\round+1}^{\learner}-p_{\round+1}^{\learner}}^2\leq \sum_{\learner=1}^\totalLearners\norm{d_{\round}^{\learner}-p_{\round}^{\learner}}^2 - \gamma^2\sum_{\learner=1}^\totalLearners\left(\loss(d_{\round}^{\learner})-\loss(p_{\round}^{\learner})\right)^2 + \divThreshold + 2\compErr^2\enspace .
	\]	
	By applying this inequality recursively for $\round=1,\dots,\totalRounds$ it follows that
	\begin{equation*}
	\begin{split}
	\sum_{\learner=1}^\totalLearners\norm{d_{\round+1}^{\learner}-p_{\round+1}^{\learner}}^2\leq& \sum_{\learner=1}^\totalLearners\norm{d_{1}^{\learner}-p_{1}^{\learner}}^2 + \totalRounds(\divThreshold+2\compErr^2)-\gamma^2\sum_{\round=1}^{\totalRounds}\sum_{\learner=1}^\totalLearners\left(\loss(d_{\round}^{\learner})-\loss(p_{\round}^{\learner})\right)^2 .
	\end{split}
	\end{equation*}
	Using $\mathbf{d}_1=\mathbf{p}_1$, we conclude that 
	\begin{equation*}
	\begin{split}
	\sum_{\round=1}^{\totalRounds}\sum_{\learner=1}^\totalLearners\left(\loss(d_{\round}^{\learner})-\loss(p_{\round}^{\learner})\right)^2\leq& \frac{1}{\gamma^2}\left(\totalRounds(\divThreshold+2\compErr^2) - \sum_{\learner=1}^\totalLearners\norm{d_{\round+1}^{\learner}-p_{\round+1}^{\learner}}^2\right)\leq \frac{1}{\gamma^2}\totalRounds\divThreshold\\
	\Leftrightarrow \cummloss{\dynProt}{\totalRounds}^{\totalLearners} - \cummloss{\periodProt}{\totalRounds}^{\totalLearners} \leq& \frac{1}{\gamma^2}\totalRounds(\divThreshold+2\compErr^2)
	\end{split}
	\end{equation*}
	\qed
\end{proof}
By setting the communication period $b=1$, this result also holds for the continuous protocol $\contProt$.

The result of Thm.~\ref{thm:lossboundModelComp} is similar to the original loss bound of the dynamic protocol but also accounts for the inaccuracy of the update rule, e.g., because of model compression. We can apply the original consistency result: if the continuous protocol is consistent, then the dynamic protocol is consistent as well. For Stochastic Gradient Descent it has been shown that the dynamic protocol is consistent for linear models~\cite{kamp2014communication}. From Thm.~\ref{thm:lossboundModelComp} follows that the dynamic protocol remains consistent for approximately loss-proportional update rules. Note that for static target distributions, consistency can be achieved by a decreasing divergence threshold and compression error, i.e., $\divThreshold_\round=\round^{\sfrac{-1}{2}}$ and $\compErr = \round^{\sfrac{-1}{4}}$.

We now provide communication bounds for the dynamic protocol.
For that, assume that the $\totalLearners$ learners maintain models in their support vector expansion. Let $\SVs_{\round}^{\learner}\subset \R^d$ denote the set of support vectors of learner $\learner\in [\totalLearners]$ at time $\round$ and $\coeffs^\learner_\round$ the corresponding coefficients. Let $B_x\in\bigO{d}$ be the number of bytes required to transmit one support vector and $B_{\coeffs}\in\bigO{1}$  be the number of bytes required for the corresponding weight. Furthermore, let $I:\N\times [\totalLearners]\rightarrow \{0,1\}$ be an indicator function that is $1$ if for learner $\learner$ at time $\round$ a new support vector has been added during the update.

We assume that a designated coordinator node performs the synchronizations, i.e., all local learners transmit their models to the coordinator which in turn sends the synchronized model back to each learner. Furthermore, we assume that all protocols apply the following trivial communication reduction strategy. 
Let $\round'$ be the time of last synchronization. Assume the coordinator stored the support vectors of the last average model $\overline{\SVs}_{\round'}$. 
Whenever a learner $\learner$ has to send its model to the coordinator, it sends all support vector coefficients $\coeffs$ but only the new support vectors, i.e., only $\SVs_{\round}^{\learner}\setminus \SVs_{\round'}^{\learner}$. This avoids redundant communication at the cost of higher memory usage at the coordinator side. In turn, after averaging the models, the coordinator sends to learner $\learner$ all support vector coefficients, but only the support vectors $\overline{\SVs}_{\round}\setminus \SVs_{\round}^{\learner}$. 

We start by bounding the communication of a continuous protocol $\contProt$, i.e., one that transmits all models from each learner in each round. The trivial communication reduction technique discussed above implies that in each round, a learner transmits its full set of support vector coefficients and potentially one support vector---depending on whether a new support vector was added in this round. Thus, at time $\round$ learner $\learner$ submits 
\begin{equation}
|\SVs_{\round}^{\learner}|B_\coeffs + I(\round,\learner)B_x
\label{eq:bytesPerLearnerPerSynchNode}
\end{equation}
bytes to the coordinator. The coordinator transmits to learner $\learner\in\totalLearners$ all support vector coefficients of the average model and all its support vectors, except the support vectors $\SVs_\round^\learner$ of the local model at learner $\learner$. Thus, it transmits the following amount of bytes.
\begin{equation}
\left|\overline{\SVs}_t\right|B_\coeffs + \left|\overline{\SVs}_t\setminus\SVs^\learner_\round\right| B_\sample = \left|\bigcup_{j=1}^\totalLearners \SVs_{\round}^{j}\right|B_\coeffs + \left|\bigcup_{j=1}^\totalLearners \SVs_{\round}^{j}\setminus \SVs_{\round}^{\learner}\right|B_\sample\enspace .
\label{eq:bytesPerLearnerPerSynchCoord}
\end{equation}
With this we can derive the following communication bound.
\begin{pr}
	The communication of the continuous protocol $\contProt$ on $\totalLearners\in\N$ learners until time $\totalRounds\in\N$ is bound by
	\begin{equation*}
	\begin{split}
	\cummcomm{\contProt}{\totalRounds,\totalLearners} \leq \totalRounds \totalLearners 2|\overline{\SVs}_{\totalRounds}|B_\coeffs + \totalLearners|\overline{\SVs}_{\totalRounds}|B_x\leq \totalLearners^2\totalRounds^2 B_\coeffs + \totalLearners^2\totalRounds B_x\in\bigO{\totalLearners^2\totalRounds^2}\enspace .
	\end{split}
	\end{equation*}
\end{pr}
\begin{proof}
The constantly synchronizing protocol transmits at each time step from each learner a set of support vector coefficients and potentially one support vector to the coordinator. The amount of bytes is given in Eq.~\ref{eq:bytesPerLearnerPerSynchNode}. The coordinator transmits the averaged model back to each learner with an amount of bytes as given in Eq.~\ref{eq:bytesPerLearnerPerSynchCoord}.
Summing up the communication over $\totalRounds\in\N$ time points and $\totalLearners$ learners yields
\begin{equation*}
\begin{split}
\cummcomm{\contProt}{\totalRounds,\totalLearners} &= \sum_{\round=1}^\totalRounds \sum_{\learner=1}^\totalLearners\left(|\SVs_{\round}^{\learner}|B_\coeffs + I(\round,\learner)B_\sample + \left|\bigcup_{j=1}^\totalLearners \SVs_{\round}^{j}\right|B_\coeffs + \left|\bigcup_{j=1}^\totalLearners \SVs_{\round}^{j}\setminus \SVs_{\round}^{\learner}\right|B_\sample\right)\\
&= \sum_{\round=1}^\totalRounds \sum_{\learner=1}^\totalLearners\left(|\SVs_{\round}^{\learner}|B_\coeffs + \left|\overline{\SVs}_{\round}\right|B_\coeffs + I(\round,\learner)B_\sample + \left|\overline{\SVs}_{\round}\setminus \SVs_{\round}^{\learner}\right|B_\sample\right)\enspace .
\end{split}
\end{equation*}
We analyze this sum separately in terms of bytes required for sending the support vectors and bytes for sending the coefficients. The amount of bytes for sending the support vectors is bounded by $ \totalLearners|\SVs_{\totalRounds}^{\learner}|B_\sample$, as we show in the following.
\begin{equation*}
\begin{split}
&\sum_{\round=1}^\totalRounds\sum_{\learner=1}^\totalLearners I(\round,\learner)B_\sample + \left|\overline{\SVs}_\round\setminus \SVs_{\round}^{\learner}\right|B_\sample = \underbrace{\sum_{\round=1}^\totalRounds\sum_{\learner=1}^\totalLearners I(\round,\learner)B_\sample }_{=|\overline{\SVs}_{\totalRounds}|B_\sample} + \sum_{\round=1}^\totalRounds\sum_{\learner=1}^\totalLearners
 \left|\overline{\SVs}_\round\setminus \SVs_{\round}^{\learner}\right|B_\sample\\
=&|\overline{\SVs}_{\totalRounds}|B_\sample+\sum_{\round=1}^\totalRounds\sum_{\learner=1}^\totalLearners \left| \left(\bigcup_{j=1}^\totalLearners \SVs_{\round}^{j}\setminus \bigcup_{j=1}^\totalLearners \SVs_{\round-1}^{j}\right)\setminus \left(\SVs_{\round}^{\learner}\setminus\overline{\SVs}_{\round-1}\right) \right|B_\sample\\
\leq& |\overline{\SVs}_{\totalRounds}| B_\sample+ \sum_{\round=1}^\totalRounds\sum_{\learner=1}^\totalLearners \sum_{\stackrel{j=1}{j\neq \learner}}^\totalLearners I(\round,\learner)B_\sample \leq |\overline{\SVs}_{\totalRounds}|B_\sample+\sum_{\round=1}^\totalRounds\sum_{\learner=1}^\totalLearners (\totalLearners-1)I(\round,\learner)B_\sample\\
\leq&|\overline{\SVs}_{\totalRounds}|B_\sample+ (\totalLearners-1)|\overline{\SVs}_{\totalRounds}|B_\sample = \totalLearners|\overline{\SVs}_{\totalRounds}|B_\sample\enspace .
\end{split}
\end{equation*}
We now bound the amount of bytes required for sending the support vector coefficients.
\begin{equation*}
\begin{split}
\sum_{\round=1}^\totalRounds\sum_{\learner=1}^\totalLearners \underbrace{|\SVs_{\round}^{\learner}|}_{\leq |\overline{\SVs}_{\totalRounds}|}B_\coeffs + \underbrace{|\overline{\SVs}_{\round}|}_{\leq |\overline{\SVs}_{\totalRounds}|}B_\coeffs  \leq \sum_{\round=1}^\totalRounds\sum_{\learner=1}^\totalLearners 2|\overline{\SVs}_{\totalRounds}|B_\coeffs =\totalRounds\totalLearners 2|\overline{\SVs}_{\totalRounds}|B_\coeffs\enspace .
\end{split}
\end{equation*}
From $\left|\overline{\SVs}_\totalRounds\right|\leq\totalLearners\totalRounds$ and the fact that we regard $B_\coeffs\in\bigO{1}$ and $B_\sample\in\bigO{d}$ as constants we can follow that 
\begin{equation*}
\begin{split}
\cummcomm{\contProt}{\totalRounds,\totalLearners} \leq 2\totalRounds\totalLearners|\overline{\SVs}_{\totalRounds}|B_\coeffs + \totalLearners|\overline{\SVs}_{\totalRounds}|B_\sample
\leq \totalLearners^2\totalRounds^2 B_\coeffs + \totalLearners^2\totalRounds B_\sample\in\bigO{\totalLearners^2\totalRounds^2}\enspace .
\end{split}
\end{equation*}
\qed
\end{proof}
Note that this communication bound implies that---unlike for linear models---synchronizing models in their support vector expansion requires even more communication than centralizing the input data. However, in real-time prediction applications, the latency induced by central computation can exceed the time constraints, rendering continuous synchronization a viable approach nonetheless.

Similarly, the communication of a periodic protocol $\periodProt$ that communicates every $b\in\N$ steps ($b$ is often referred to as mini-batch size) can be bounded by 
\[
\cummcomm{\periodProt}{\totalRounds,\totalLearners} \leq \frac{\totalRounds}{b}2\totalLearners |\overline{\SVs}_{\totalRounds}|B_\coeffs + \totalLearners|\overline{\SVs}_{\totalRounds}|B_x\leq \frac{\totalRounds}{b}\totalLearners^2 \totalRounds B_\coeffs + \totalLearners^2\totalRounds B_x\in\bigO{\frac{1}{b}\totalLearners^2\totalRounds^2}\enspace .
\]

We now for the first time provide a communication bound for the dynamic protocol $\dynProt$. 
For that, we first bound the number of synchronization steps and then analyze the amount of communication per synchronization. 
\begin{pr}
Let $\onlineAlgo=(\hilbertSpace,\widetilde{\uprule}, \loss)$ be an online learning algorithm with an approximately loss-proportional convex update rule $\widetilde{\uprule}$ for which holds that\break ${\|\model-\widetilde{\uprule}(\model,\sample,\truelabel)\|\leq\eta\loss(\model,\sample,\truelabel)}$.
The number of synchronizations $V_\dynProt(\totalRounds)$ of the dynamic protocol $\dynProt$ running $\onlineAlgo$ in parallel on $\totalLearners$ nodes until time $\totalRounds\in\N$ with divergence threshold $\divThreshold$ is bounded by
\[
V_\dynProt(\totalRounds) \leq \frac{\eta}{\sqrt{\divThreshold}}\cummloss{\dynProt}{\totalRounds,\totalLearners}\enspace .
\]
where $\cummloss{\dynProt}{\totalRounds,\totalLearners}$ denotes the cumulative loss of $\dynProt$.
\label{prop:numberViolations}
\end{pr} 
\begin{proof}
For this proof, we abbreviate $\loss(\model_{\round}^{\learner},\sample_\round^\learner,\truelabel_\round^\learner)$ as $\loss(\model_{\round}^{\learner})$ and $\widetilde{\uprule}(\model_{\round}^{\learner},\sample_\round^\learner,\truelabel_\round^\learner)$ as $\widetilde{\uprule}(\model_{\round}^{\learner})$. The dynamic protocol synchronizes if a local condition $\|\model_{\round}^{\learner}-\refModel_\round\|^2\leq\divThreshold$ is violated. Now assume that at $\round=1$ all models are initialized with $\model_{1}^{1}=\dots=\model_{1}^{\totalLearners}$ and $\refModel_1=\avgmodel_1$, i.e., for all local learners $\learner$ it holds that $\|\model_{1}^{\learner}-\refModel_1\|=0$. A violation, i.e., $\|\model_{\round}^{\learner}-\refModel_\round\|>\sqrt{\divThreshold}$, occurs if one local model drifts away from $\refModel_\round$ by more than $\sqrt{\divThreshold}$. After a violation, a synchronization is performed and $\refModel_\round=\avgmodel_\round$, hence $\|\model_{\round}^{\learner}-\refModel_\round\|=0$ and the situation is again similar to the initial setup for $\round=1$. In the worst case, a local learner drifts continuously in one direction until a violation occurs. Hence, we can bound the number of violations $V_\learner(\totalRounds)$ at a single learner $\learner$ by the sum of its drifts divided by $\sqrt{\divThreshold}$:
\begin{equation*}
\begin{split}
V_l(\totalRounds)\leq \frac{1}{\sqrt{\divThreshold}}\sum_{\round=1}^\totalRounds\|\model_{\round}^{\learner}-\model_{\round+1}^{\learner}\|=&\frac{1}{\sqrt{\divThreshold}}\sum_{\round=1}^\totalRounds\underbrace{\|\model_{\round}^{\learner}-\widetilde{\uprule}(\model_{\round}^{\learner})\|}_{\leq\eta\loss(\model_{\round}^{\learner})}
\leq\frac{1}{\sqrt{\divThreshold}}\sum_{\round=1}^\totalRounds \eta\loss(\model_{\round}^{\learner})\enspace .
\end{split}
\end{equation*}
With this, we can bound the amount of points in time $t\in [\totalRounds]$ where at least one learner $l$ has a violation, i.e., $V(\totalRounds)$. In the worst case, all violations at all local learners occur at different time points, so that we can upper bound $V(\totalRounds)$ by the sum of local violations $V_\learner(\totalRounds)$ which is again upper bounded by the cumulative sum of drifts of all local models:
\[
V(\totalRounds)\leq \sum_{\learner=1}^\totalLearners V_\learner(\totalRounds) \leq \frac{1}{\sqrt{\divThreshold}}\sum_{\round=1}^\totalRounds\sum_{\learner=1}^\totalLearners \eta\loss(\model_{\round}^{\learner})=\frac{\eta}{\sqrt{\divThreshold}}\cummloss{\dynProt}{\totalRounds,\totalLearners}\enspace .
\]
\qed
\end{proof}
In the following theorem we bound the overall communication by combining this bound on the number of synchronizations with an analysis of the amount of bytes transfered per synchronization.
\begin{thm}
Let $\onlineAlgo=(\hilbertSpace,\widetilde{\uprule}, \loss)$ be an online learning algorithm with approximately loss-proportional update rule $\widetilde{\uprule}$ and $\|\model-\widetilde{\uprule}(\model,\sample,\truelabel)\|\leq\eta\loss(\model,\sample,\truelabel)$.
The amount of communication $C_\dynProt(\totalRounds,\totalLearners)$ of the dynamic protocol $\dynProt$ running $\onlineAlgo$ in parallel on $\totalLearners$ nodes until time $\totalRounds\in\N$ with divergence threshold $\divThreshold$ is bounded by
\begin{equation*}
\begin{split}
C_\dynProt(\totalRounds,\totalLearners)\leq \frac{\eta}{\sqrt{\divThreshold}}\cummloss{\dynProt}{\totalRounds,\totalLearners}\left(2\totalLearners \left|\overline{\SVs}_\totalRounds\right| B_\coeffs \right) + \totalLearners \left|\overline{\SVs}_{\totalRounds}\right| B_\sample
\end{split}
\end{equation*}
\end{thm}
\begin{proof}
Assume that at time $\totalRounds$, the dynamic protocol performs a synchronization. Then, similar to the argument for the continuous protocol, the support vector set at time $\totalRounds$ is similar for all learners and independent of the number of synchronization steps before. In particular, it is the same if a synchronization was performed in every time step. Thus, again the amount of bytes required for sending the support vectors is bounded by $\totalLearners \left|\overline{\SVs}_\totalRounds\right| B_\sample$. 
Let ${\mapping{\theta}{\N}{\{0,1\}}}$ be an indicator function such that $\theta(\round)=1$ if at time $\round$ the dynamic protocol performed a synchronization and $\theta(\round)=0$ otherwise. Then, the amount of bytes required to send all the support vector coefficients until time $\totalRounds$ is
\begin{equation*}
\begin{split}
\sum_{\round=1}^\totalRounds \theta(\round) \sum_{\learner=1}^\totalLearners\left( \left|\SVs_{\round}^{\learner}\right| + \left|\overline{\SVs}_t\right|\right)B_\coeffs\leq &\underbrace{\sum_{t=1}^T\theta(t)}_{= V_\dynProt(\totalRounds)}\sum_{l=1}^\totalLearners2|\overline{\SVs}_{\totalRounds}| B_\coeffs
\leq \underbrace{\frac{\eta}{\sqrt{\divThreshold}}\cummloss{\dynProt}{\totalRounds,\totalLearners}}_{\text{Prop.~\ref{prop:numberViolations}}}\left( 2\totalLearners |\overline{\SVs}_{\totalRounds}| B_\coeffs\right)\\
\end{split}
\end{equation*}
Together with the amount of bytes required for exchanging all support vectors this yields
$
C_\dynProt(\totalRounds,\totalLearners)\leq \frac{\eta}{\sqrt{\divThreshold}}\cummloss{\dynProt}{\totalRounds,\totalLearners}\left( 2\totalLearners |\overline{\SVs}_{\totalRounds}| B_\coeffs\right) + \totalLearners\left|\overline{\SVs}_\totalRounds\right| B_\sample
$.
\qed
\end{proof}
Note that the loss bounds for online learning algorithms are typically sub-linear in $\totalRounds$, e.g., optimal regret bounds for static target distributions are in $\mathcal{O}(\sqrt{\totalRounds})$. In these cases, the dynamic protocol has an amount of communication in $\mathcal{O}(\totalLearners^2\totalRounds\sqrt{\totalRounds})$ which is smaller than $\mathcal{O}(\totalLearners^2\totalRounds^2)$ of the continuously and periodic protocols by a factor of $\sqrt{\totalRounds}$. 

In the original case of linear models instead, the dynamic protocol only transmits $\totalLearners$ weight vectors of fixed size per synchronization. In this case the amount of communication per synchronization is bounded by a constant. If for an online learning algorithm $\onlineAlgo$ and the periodic protocol $\periodProt$ it holds that $\lossbound{\periodProt}{\totalRounds,\totalLearners}\leq\lossbound{\onlineAlgo}{\totalLearners\totalRounds}$, then by Thm.~\ref{thm:lossboundModelComp} it also holds that $\lossbound{\dynProt}{\totalRounds,\totalLearners}\leq\lossbound{\onlineAlgo}{\totalLearners\totalRounds}$. This implies that the dynamic protocol is adaptive. In the following corollary, we show that for linear models, the dynamic protocol is adaptive when using the Stochastic Gradient Descent algorithm.
\begin{crl}
The dynamic protocol $\dynProt=(\text{SGD},\syncop_\divThreshold)$ using Stochastic Gradient Descent $\text{SGD}$ with linear models is adaptive, i.e., 
\[
\cummcomm{\dynProt}{\totalRounds,\totalLearners}\in \bigO{ \totalLearners \lossbound{\text{SGD}}{\totalLearners\totalRounds} }
\]
\end{crl}
\begin{proof}
The amount of synchronizations of the dynamic protocol is bounded by $V(\totalRounds)$ (see Prop.~\ref{prop:numberViolations}). In each synchronization, each learner transmits one linear model, i.e., one weight vector of fixed size to the coordinator. The coordinator submits one averaged weight vector back to each learner. Thus, the amount of communication per synchronization is bounded by $\mathbf{c}_\totalLearners\in\N$, where $\mathbf{c}_\totalLearners\in\bigO{\totalLearners}$. Then, the total communication is bounded by
\[
\cummcomm{\dynProt}{\totalRounds,\totalLearners}\leq \mathbf{c}_\totalLearners\frac{\eta}{\sqrt{\divThreshold}}\cummloss{\dynProt}{\totalRounds,\totalLearners}\in\bigO{\totalLearners \cummloss{\dynProt}{\totalRounds,\totalLearners}}\enspace .
\]
The dynamic protocol retains the loss bound of Stochastic Gradient Descent~\cite{kamp2014communication}, i.e., $\cummloss{\dynProt}{\totalRounds,\totalLearners}\leq \lossbound{\text{SGD}}{\totalLearners\totalRounds}$
\qed
\end{proof}
Unfortunately, from Thm.~\ref{thm:lossboundModelComp} also follows that the dynamic protocol applied to kernelized online learning algorithms that do not bound the size of their models does not comply to the strict notion of adaptivity as given in Def.~\ref{def:efficiency}. That is, because
the model size and thus the size of each message to and from the coordinator can grow with $\totalRounds$. Nonetheless, the theorem guarantees that if the learners do not suffer loss anymore, the dynamic protocol reaches quiescence.

In order to make the dynamic protocol adaptive in the strict sense of Def.~\ref{def:efficiency}, the model size has to be bounded. For kernelized online learning in streams, several \defemph{model compression} techniques have been proposed~\cite{kivinen2004online,orabona2009bounded,wang2010online}. These techniques typically guarantee that the compression error is bounded, i.e., for the \defemph{compressed model} $\modelApprox$ it holds that $\norm{\model-\modelApprox}\leq\compErr$. From this directly follows that if the base algorithm uses a loss-proportional convex update rule $\uprule$, the compressed version is an approximately loss-proportional convex update rule $\widetilde{\uprule}$.

One approach to compressing the support vector expansion is to project a new support vector on the span of the remaining ones and thus avoid adding it to the support set. Another one is to truncate support vectors with small coefficients. 
For the projection approach (e.q., described in~\cite{orabona2009bounded}) the error bound is independent of the learning algorithm. However, there is no bound on the number of support vectors. Thus, even though the model size is reduced in practice, there is no formal bound on the model size. For the truncation approach, however,~\cite{kivinen2004online} have shown that an error bound as well as a bound on the number of support vectors can be achieved when using Stochastic Gradient Descent. Specifically, for a fixed model size of $\tau$ support vectors, they have shown that the compression error is bound by
$
\norm{\model-\modelApprox}\leq\compErr\in\bigO{\frac{1}{\lambda}(1-\lambda)^\tau}
$,
where $\lambda\in\R$ is the learning rate of the Stochastic Gradient Descent algorithm (SGD). Therefore, we can follow that the dynamic protocol with SGD using kernel models compressed by truncation is adaptive. Specifically for SGD, \cite{dekel/jmlr/2012} have shown that periodic synchronizations retain the serial loss bound of SGD. It is consistent in this setting, because the dynamic protocol in turn retains the loss bounds of any periodic protocol. Since it is both consistent and adaptive, the dynamic protocol is efficient.

\section{Discussion}
\label{sec:discussion}
\begin{figure*}[t]
	\centering
	\subfigure[]{\label{fig:financePerformance}\includegraphics[height=4.23cm]{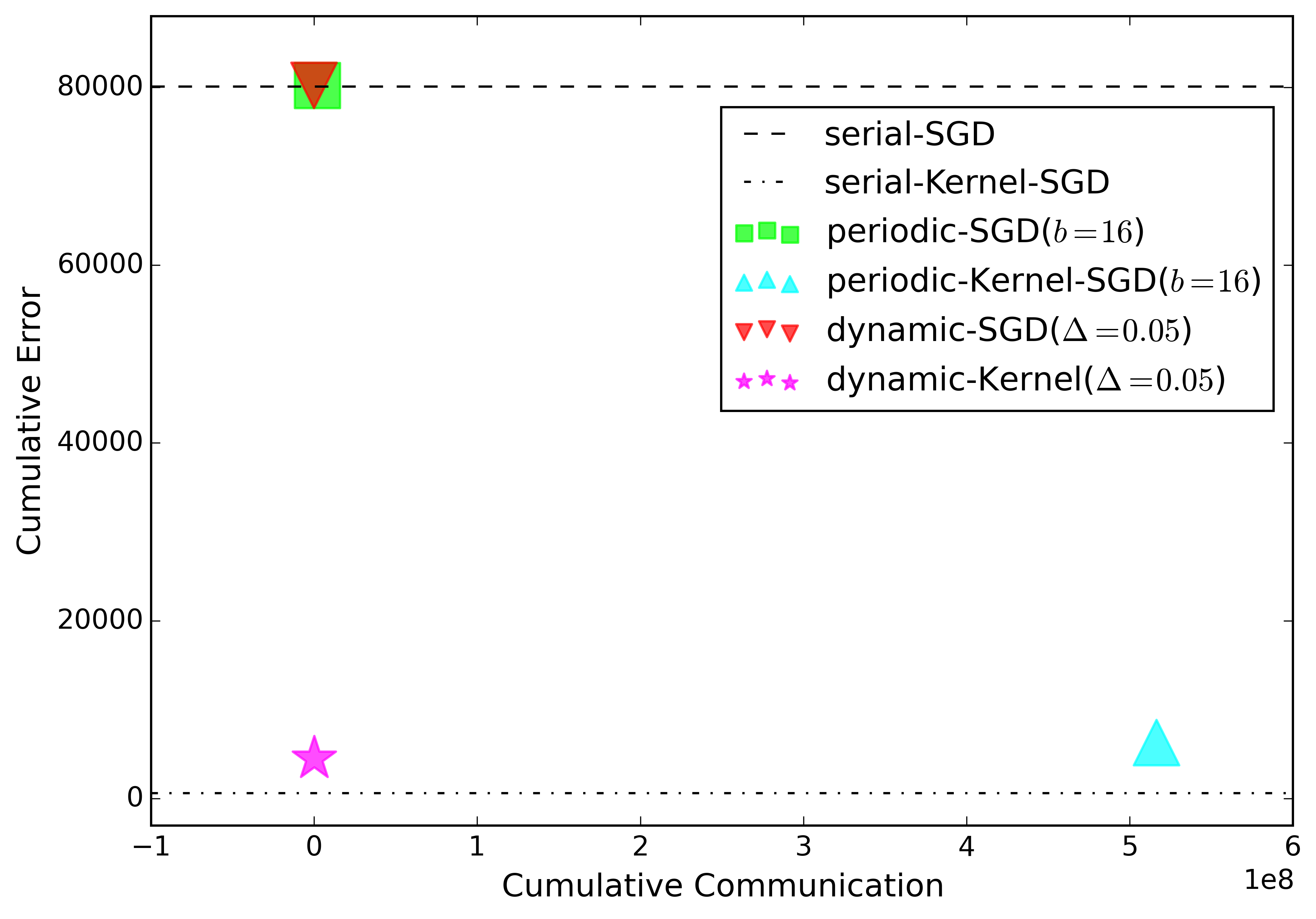}}\hfill%
	\subfigure[]{\label{fig:financeCommunication}\includegraphics[height=4.3cm]{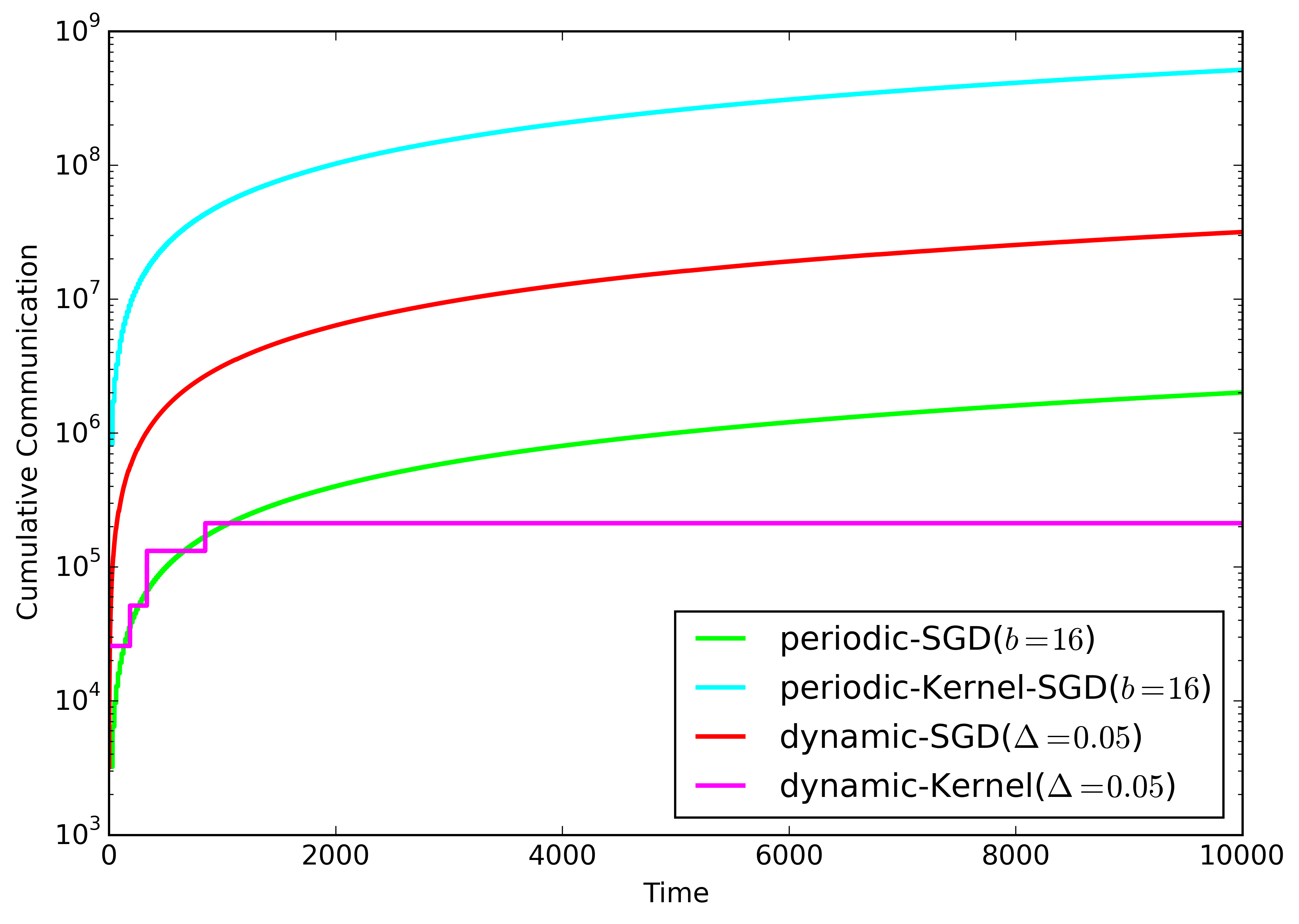}}
	\caption{
		\subref{fig:financePerformance} Trade-off between cumulative error and cumulative communication and~
		\subref{fig:financeCommunication} cumulative communication over time of the dynamic protocol versus a periodic protocol. $32$ learners perform a stock price prediction task using SGD (learning rate $\eta$ and regularization parameter $\lambda$ optimized over $200$ instances, with $\eta=10^{-10}$, $\lambda=1.0$ for the periodic protocol, and $\eta=1.0$, $\lambda=0.01$ for the dynamic protocol) updates, either with linear models or with non-linear models (Gaussian kernel with number of support vectors limited to $50$ using the truncation approach of~\cite{kivinen2004online}).
	}
	\label{fig:financeExp}
\end{figure*}
The dynamic protocol, extended to kernel methods, yields for the first time a theoretically efficient tool to learn non-linear models for distributed real-time services, in settings where communication is a major bottleneck. For that, it can employ online kernel methods together with model compression techniques, which reduce, or bound the number of support vectors. 
The efficiency of the protocol is characterized by a novel criterion that ties a tight loss bound to the required amount of communication---a criterion which is not satisfied by the state of the art of periodically communicating protocols.

While we provided a theoretical analysis, the advantage of the dynamic protocol in combination with kernel methods can also be shown in practice: Fig.~\ref{fig:financeExp} shows the results of an experiment on financial data~\cite{kamp2013beating}, where $32$ learners predicted the stock price of a target stock. We can see that for this difficult learning task linear models perform poorly compared to non-linear models using a Gaussian kernel function. Simultaneously, the communication required to periodically synchronize these non-linear models is larger than for linear models by more than two orders of magnitude. Using the dynamic protocol with kernel models we could reduce the error by an order of magnitude compared to using linear models (a reduction by a factor of $18$). At the same time, the communication is reduced by more than three orders of magnitude compared to the static protocol (by a factor of $2433$), which is yet an order of magnitude smaller than the communication when using linear models (by a factor of $10$). Moreover, within less than $2000$ rounds, the dynamic protocol reaches quiescence, as it is implied by the efficiency criterion.

A limit of the employed notion of efficiency is that it only takes into account the sum of messages but not the peak communication. In large data centers, where the distributed learning system is run next to other processes, the main bottleneck is the overall amount of transmitted bytes and a high peak in communication can often be handled by the communication infrastructure or evened out by a load balancer. In smaller systems, however, high peak communication can become a serious problem for the infrastructure and it remains an open problem how it can be reduced. Note that the frequency of synchronizations in a short time interval can actually be bounded by a trivial modification of the dynamic protocol: local conditions are only checked after a mini-batch of examples have been observed. Thus, the peak communication is upper bounded in the same way as with a periodic protocol, while still dynamically reducing the overall amount of communication.

When analyzing the reason for practical efficiency, model compression has proven to be a crucial factor, since storing and evaluating models with large numbers of support vectors can become infeasible---even in serial settings. In a distributed setting, transmitting large models furthermore induce high communication costs, which is aggravated by averaging local models, because the synchronized model consists of the union of all local support vectors.
For the model truncation approach of~\cite{kivinen2004online}, we have shown that the efficiency criterion is satisfied, but other model compression approaches might be favorable in certain scenarios.
Thus, an interesting direction for future research is to study the relationship between loss and model size of those model compression techniques in order to extend the results on efficiency.

Also, alternative approaches to ensuring constant model size could be investigated, e.g., a finite dimensional approximation of the feature map $\mapping{\Phi}{\sampleSpace}{\hilbertSpace}$ of a reproducing kernel Hilbert space $\hilbertSpace$, such as Random Fourier Features~\cite{rahimi2007random}. It remains an open problem how tight loss bounds combined with communication bounds can be derived in these settings.

Finding the right divergence threshold for the dynamic protocol, i.e., one that suits the desired trade-off between service quality and communication, is in practice a neither intuitive nor trivial task. The threshold can be selected using a small data sample, but the communication for a given threshold can vary over time and is also influenced by other parameters of the learner. Thus, another direction for future research is to investigate an adaptive divergence threshold. This could allow for a more direct selection of the desired trade-off between service quality (i.e., predictive performance) and communication.

\subsubsection*{Acknowledgments}
This research has been supported by the EU FP7-ICT-2013-11 under grant 619491 (FERARI).\pagebreak
\bibliographystyle{splncs03}
\bibliography{bibliography}
\end{document}